\documentclass[conference]{IEEEtran}
\IEEEoverridecommandlockouts

\usepackage{amsmath,amssymb,amsfonts,amsthm}
\usepackage{algorithmic}
\usepackage{graphicx}
\usepackage{textcomp}
\usepackage{xcolor}
\usepackage{cleveref}
\def\BibTeX{{\rm B\kern-.05em{\sc i\kern-.025em b}\kern-.08em
    T\kern-.1667em\lower.7ex\hbox{E}\kern-.125emX}}

\usepackage[backend=biber,style=numeric-comp,sorting=none]{biblatex}
\newcommand{\Xdata}{\mathbb X}
\newcommand*{\defeq}{\mathrel{\vcenter{\baselineskip0.5ex\lineskiplimit0pt\hbox{\scriptsize.}\hbox{\scriptsize.}}}%
	=}

\renewcommand{\vec}[1]{\boldsymbol{#1}}
\newcommand{\mat}[1]{\boldsymbol{#1}}

\DeclareMathOperator*{\diag}{diag}

\DeclareMathOperator*{\argmin}{arg\,min}

\newtheorem{theorem}{Theorem}

\newtheorem{proposition}{Proposition}
\newtheorem{corollary}{Corollary}

\crefname{definition}{Def.}{Defs.}
\crefname{equation}{Eq.}{Eqs.}
\crefname{section}{Sec.}{Secs.}
\crefname{theorem}{Theorem}{Theorems}
\crefname{proposition}{Prop.}{Props.}

\begin{document}

\title{Kernel $k$-Medoids as General Vector Quantization 
}

\author{\IEEEauthorblockN{1\textsuperscript{st} Thore Gerlach}
\IEEEauthorblockA{\textit{University of Bonn}\\
Bonn, Germany \\
gerlach@iai.uni-bonn.de}
\and
\IEEEauthorblockN{2\textsuperscript{nd} Sascha M\"ucke}
\IEEEauthorblockA{\textit{TU Dortmund University}\\
	Dortmund, Germany \\
	sascha.muecke@tu-dortmund.de}
\and
\IEEEauthorblockN{3\textsuperscript{rd} Christian Bauckhage}
\IEEEauthorblockA{\textit{Fraunhofer IAIS}\\
Sankt Augustin, Germany \\
christian.bauckhage@iais.fraunhofer.de}}


\maketitle

\begin{abstract}
Vector Quantization~(VQ) is a widely used technique in machine learning and data compression, valued for its simplicity and interpretability. 
Among hard VQ methods, $k$-medoids clustering and Kernel Density Estimation~(KDE) approaches represent two prominent yet seemingly unrelated paradigms---one distance-based, the other rooted in probability density matching.
In this paper, we investigate their connection through the lens of Quadratic Unconstrained Binary Optimization (QUBO).
We compare a heuristic QUBO formulation for $k$-medoids, which balances centrality and diversity, with a principled QUBO derived from minimizing Maximum Mean Discrepancy in KDE-based VQ.
Surprisingly, we show that the KDE-QUBO is a special case of the $k$-medoids-QUBO under mild assumptions on the kernel's feature map.
This reveals a deeper structural relationship between these two approaches and provides new insight into the geometric interpretation of the weighting parameters used in QUBO formulations for VQ.
\end{abstract}


\section{Introduction}
\label{sec:intro}

Vector Quantization~(VQ)~\cite{gersho2012vector,kohonen2001learning} is a foundational method in machine learning and data compression, where the goal is to represent a large dataset using a significantly smaller subset of representative points known as prototypes.
In hard VQ, the prototypes are selected directly from the dataset itself, ensuring that all representations are grounded in real, observed data.
These methods are particularly valued for their intuitive problem formulation and interpretable model behavior~\cite{lisboa2023coming}---traits that are increasingly important in domains such as healthcare~\cite{vellido2020importance}, autonomous systems~\cite{kim2017interpretable} or scientific discovery~\cite{tshitoyan2019unsupervised}, where transparency and explainability are essential.
Depending on the context, prototypes may serve different purposes: in unsupervised learning, they are used for data representation or clustering~\cite{cheng1995mean}, while in supervised settings, they support classification or regression tasks~\cite{macqueen1967some}.

Among the many approaches to VQ, two families of methods have emerged as particularly prominent: distance-based centroid selection and probability density estimation methods.
The former includes algorithms like $k$-means~\cite{lloyd1982least} and $k$-medoids~\cite{kaufman1987clustering}, where the goal is to identify a set of data points that minimize the average dissimilarity to the rest of the dataset.
The latter group includes Kernel Density Estimation~(KDE) approaches, where the similarity between the data distribution and the selected prototypes is optimized using tools like the kernel trick and information-theoretic concepts~\cite{rao2007information,xu2008reproducing}.
While both $k$-medoids-based VQ~(MED-VQ) and KDE-based VQ~(KDE-VQ) can be interpreted as prototype selection problems, their conceptual and mathematical connections remain poorly understood.
They appear, at first glance, to address distinct optimization goals---centrality versus distribution matching---and have evolved largely independently within the literature.

A significant challenge shared by hard VQ methods is their combinatorial nature.
Finding optimal prototype sets is computationally intractable in general, with problems like $k$-medoids being NP-hard~\cite{aloise2009np}.
As the search space grows exponentially with the dataset size, exact solutions become impractical for real-world applications.
This computational bottleneck has motivated the use of specialized hardware accelerators, particularly in solving Quadratic Unconstrained Binary Optimization~(QUBO)~\cite{punnen2022quadratic} problems that arise in prototype selection, which are of the form
\begin{equation}\label{eq:qubo}
	\min_{\vec z\in\{0,1\}^n} \vec z^{\top}\mat Q\vec z\;,
\end{equation}  
where $\mat Q$ is an $n\times n$ real matrix.
Among these hardware accelerators, Ising machines~\cite{mohseni2022ising} represent a notable direction. 
These devices are designed to solve optimization problems of the form given in~\cref{eq:qubo} by finding the ground state of an Ising Hamiltonian, effectively translating optimization problems into physical energy minimization processes.
Quantum Computing~(QC)~\cite{nielsen2010quantum}, in particular, offers a powerful model in this realm by exploiting superposition, entanglement, and quantum tunneling.
One can differentiate between digital QC---coined Quantum Gate Computing~(QGC)---which is used to implement sequential algorithmic structures via unitary operations and analog QC, such as Quantum Annealing~(QA)~\cite{kadowaki1998quantum}, based on the principles of Adiabatic Quantum Computing~(AQC)~\cite{albash2018adiabatic}.
While fully scalable quantum computers remain in development, early devices---classified as Noisy Intermediate-Scale Quantum~(NISQ) hardware~\cite{preskill2018quantum}---are already accessible and have spurred growing interest in quantum optimization applied to a diverse range of real-world problems~\cite{piatkowski2022towards,gerlach2024fpga2,mucke2025optimum,gerlach2025hybrid,gerlach2025dynamic}.

In this paper, we investigate the relationship between two QUBO formulations of hard VQ: one derived from $k$-medoids~\cite{bauckhage2019} clustering, and another from KDE divergence minimization~\cite{bauckhage2020}.
The QUBO formulation for MED-VQ follows rather heuristic arguments where two competing objectives---the selection of central and mutually distant data points---are balanced via scalar weighting parameters.
In opposition, a KDE-VQ QUBO is derived with a principled approach by minimizing the Maximum Mean Discrepancy~(MMD) between the full dataset and a candidate prototype subset.
Both VQ methods are first formulated as Quadratic Binary Programming~(QBP)~\cite{laughhunn1970} problems and then reformulated to QUBOs by incorporating linear constraints.
Surprisingly, we find that the principled QUBO based on MMD minimization is a special case of the heuristic MED-VQ approach.
This equivalence emerges under a mild assumption: that the kernel-induced feature map embeds data onto the unit sphere of the feature space---an assumption that holds for commonly used Radial Basis Function~(RBF) or Laplacian kernels as well as quantum kernels \cite{havlivcek2019supervised}.
Our contributions are threefold:
\begin{itemize}
	\item We show that the KDE-VQ QUBO formulation is a special case of the MED-VQ QUBO, revealing a deeper structural relationship between distance-based and density-based VQ,
	\item By showing that the derived equivalence holds for normalized Mercer kernels (including RBF, Laplacian, and quantum kernels), the paper highlights a broader applicability of the findings. This generalization opens avenues for combining the AQC and QGC paradigms,
	\item We interpret the role of the scalar weighting parameter in the MED-VQ QUBO formulation as a scaling factor for the kernel-induced geometry, providing insights into its meaning for quantum-based optimization.
\end{itemize}
This paper is structured as follows: we first give some insights on related work in~\cref{sec:related}, while~\cref{sec:background} captures all necessities for obtaining QUBO formulations for MED-VQ and KDE-VQ.
In~\cref{sec:discussion}, we state our main theoretical insights and discuss them in~\cref{sec:discussion}.
Finally, we draw a conclusion in~\cref{sec:conclusion}.

\section{Related Work}
\label{sec:related}

A growing body of research explores how QC and quantum-inspired models can enhance or accelerate VQ.
In \cite{villmann2021quantum}, the authors integrate VQ into a hybrid framework that leverages quantum-inspired models to optimize prototype adaptation. 
In \cite{engelsberger2022steps}, the authors take a step further by directly formulating learning VQ~(LVQ) for classification tasks on theoretical QGC architectures. 
The work in \cite{villmann2020quantum} explores similar ideas but focuses on how quantum principles, such as unitary operations and kernel-based feature embeddings, can be used to improve prototype-based classifiers.

Another promising direction involves leveraging AQC and related QUBO-based formulations to address the combinatorial nature of prototype selection in VQ.
The foundation stone was layed in~\cite{bauckhage2019}, where the authors come up with a QUBO formulation for the $k$-medoids problem, by balancing two competing objectives.
This work is extended in~\cite{sijpesteijn2023quantum}, where a preferable choice of the balancing parameters is proposed.
A different approach is taken in~\cite{engelsberger2023quantum}, where VQ is interpreted as a Set Cover problem, which is then reformulated into the QUBO framework.
Finally, the authors of~\cite{bauckhage2020} utilize MMD for KDE to obtain a formulation suited for Hopfield networks, which are equivalent to QUBOs. 

\section{Background}
\label{sec:background}

For notational convenience, we indicate vectors by bold lowercase and matrices by bold capitalized letters.
Further, we define $\mathbb B\defeq\{0,1\}^n$ and denote the vector consisting of only ones by $\vec 1$, where its dimension is induced by the context.
Moreover, the function $\diag[\cdot]$, which takes a vector as input, denotes a diagonal matrix with that vector as its diagonal.

Generally speaking, VQ is a technique used in signal processing and ML to compress a dataset into a finite set of representative points, known as prototypes.
For $\mathcal D=\{\vec x^1,\dots,\vec x^n\}\subset\Xdata$, it aims to identify $k\ll n$ prototypes $\mathcal W=\{\vec w^1,\dots,\vec w^k\}\subset\Xdata$ that serve as a compressed representation of the dataset.
These representative points form a codebook, and each input vector is mapped to its closest prototype, effectively quantizing the data.

It functions by partitioning a large set of points (vectors) into groups, with each group containing roughly the same number of points nearest to it.
This makes VQ suitable for lossy data compression, pattern recognition and density estimation.
Furthermore, it is very similar to clustering but describes a slightly different task.

\subsection{$k$-Medoids}
\label{sec:kmedoids}

A simple VQ method is to use the cluster means of $k$-means clustering as prototypes.
In hard VQ, the prototypes exactly correspond to the medoids of the clusters.

We focus on hard clustering, i.e., we want to find disjoint subsets $C_i\subset\mathcal D$ with $\bigcup_i C_i=\mathcal D$ such that points in $C_i$ are similar and points from two different clusters $C_i$ and $C_j$ are dissimilar.
Hence, one relies on suitable similarity measures such as Euclidean distance or cosine similarity.
The $k$-medoids objective is to minimize the within cluster scatter
\begin{equation}\label{eq:kmeans_objective}
	\min_{C_1,\dots,C_k}\ \sum_{i\in[k]}\sum_{\vec x\in C_i}D(
	 \vec x,\vec m_i)\;,
\end{equation}
where the medoid $\vec m$ is defined to be an element of $C$
\begin{equation*}
	\vec m=\argmin_{\vec{y}\in C}\sum_{\vec x\in C}D(\vec x,\vec y)\;,
\end{equation*}
and $D:\Xdata\times\Xdata\to \mathbb R$ is a suitable distance measure.
A notable characteristic of medoids is that they are determined solely by evaluating distances between given data points.
Such distances can be precomputed and thus do not rely on numeric data, in opposition to means.

A local minimum of the objective in~\cref{eq:kmeans_objective} can be obtained by using a slightly different versions of Lloyd's algorithm~\cite{lloyd1982least}.
This method initializes $k$ medoids and iteratively determines clusters by assigning data points to their closest medoid and updates these medoids according to that assignment.
Finding a global optimum, however, is much harder, since $k$-medoids is NP-hard \cite{aloise2009np}.

The authors of \cite{bauckhage2019} propose a QUBO formulation for estimating \cref{eq:kmeans_objective}.
For this, \emph{Fisher’s analysis of variance} \cite{fisher1921probable} is used, which states that the sum of the within cluster scatter and the between cluster scatter is constant.
Minimizing the within cluster scatter can be written as selecting the $k$ most central data points
\begin{align*}
	\min_{\mathcal W\subset \mathcal D}&\ \sum_{\vec x\in\mathcal W}\sum_{\vec y\in\mathcal D}D(\vec x,\vec y) \\
	\text{s.t.}&\ |\mathcal W|=k\;,
\end{align*}
while maximizing the between cluster scatter is similar to finding far apart data points
\begin{align*}
	\max_{\mathcal W\subset \mathcal D}&\ \sum_{\vec x\in\mathcal W}\sum_{\vec y\in\mathcal W}D(\vec x,\vec y) \\
	\text{s.t.}&\ |\mathcal W|=k\;,
\end{align*}
where $\mathcal W=\{\vec m_1,\dots,\vec m_k\}$ denotes the set of cluster medoids.
Putting both objectives together by weighting them with parameters $\alpha,\beta>0$, we get a single objective
\begin{align*}
	\min_{\mathcal W\subset \mathcal D}&\ -\alpha\sum_{\vec x\in\mathcal W}\sum_{\vec y\in\mathcal W}D(\vec x,\vec y) +\beta \sum_{\vec x\in\mathcal W}\sum_{\vec y\in\mathcal D}D(\vec x,\vec y)\\
	\text{s.t.}&\ |\mathcal W|=k\;.
\end{align*}
Since finding the minimum is invariant to scaling, we obtain an equivalent formulation
\begin{subequations}
	\begin{align*}
		\min_{\mathcal W\subset \mathcal D}&\ -\sum_{\vec x\in\mathcal W}\sum_{\vec y\in\mathcal W}D(\vec x,\vec y) +\gamma \sum_{\vec x\in\mathcal W}\sum_{\vec y\in\mathcal D}D(\vec x,\vec y)\\
		\text{s.t.}&\ |\mathcal W|=k\;,
	\end{align*}
\end{subequations}
by setting $\gamma\defeq\frac{\beta}{\alpha}$.
The parameter $\gamma$ then weighs the objectives, with a small value preferring the far away objective, while larger values lead to more centralized data points.

Since $\mathcal W\subset\mathcal D$, we can indicate the membership of the data point $\vec x_i$ in $\mathcal W$ by a binary variable $z_i$.
This leads to a QBP formulation by using the distance matrix $\mat D=(D(\vec x,\vec y))_{\vec x,\vec y\in\mathcal D}$
\begin{subequations}
\label{eq:kmedoids_qbp}
\begin{align}
	\min_{\vec z\in\mathbb B^n}&\ - \vec z^{\top}\mat D\vec z +\gamma \left(\mat D\vec 1\right)^{\top}\vec z\\
	\text{s.t.}&\ \vec 1^{\top}\vec z=k\;.
\end{align}
\end{subequations}

A QUBO formulation is obtained \cite{bauckhage2019} by integrating the quadratic penalty function $(\vec 1^{\top}\vec z-k)^2$ into the objective
\begin{equation}
	\min_{\vec z\in\mathbb B^n}\ - \vec z^{\top}\mat D\vec z +\gamma \left(\mat D\vec 1\right)^{\top}\vec z+\lambda_{\text{MED}} \left(\vec 1^{\top}\vec z-k\right)^2\;,
	\label{eq:kmedoids_qubo}
\end{equation}
where equivalence to \cref{eq:kmedoids_qbp} holds for large enough $\lambda_{\text{MED}} >0$.
Finding a suitable value is far from trivial and is of great interest in current research~\cite{alessandroni2023alleviating}.

Due to $z_iz_i=z_i$ for $z_i\in\mathbb B^n$, the linear offset terms in~\cref{eq:kmedoids_qubo} can be absorbed into a quadratic objective, by adding the corresponding vectors on the diagonal of the quadratic matrix.
In other words, the QUBO matrix (see \cref{eq:qubo}) of \cref{eq:kmedoids_qubo} is given by
\begin{equation}
	\mat{Q}_{\text{MED}}=-\mat D +\lambda_{\text{MED}}\vec 1\vec 1^{\top}+\diag\left[\gamma \mat D\vec 1-2\lambda_{\text{MED}}k\vec 1\right]\;.
	\label{eq:kmedoids_qubo_matrix}
\end{equation}

\subsection{Kernel Density Estimation}

As VQ aims to model the underlying probability density functions of the data distribution with the help of the codebook, we also consider probability density estimates. 
That is, assuming that $p_{\mathcal D}(\cdot)$ is an underlying probability density function of $\mathcal D$ and $p_{\mathcal W}(\cdot)$ of $\mathcal W$, we want $p_{\mathcal D}(\cdot)$ and $p_{\mathcal W}(\cdot)$ to be as similar as possible.
The dissimilarity or \emph{divergence} can be measured in numerous ways, examples include the \emph{Cauchy-Schwartz} divergence~\cite{jenssen2006cauchy} or the \emph{Kullback-Leibler} divergence~\cite{kullback1951information}.

Since we have no knowledge on the underlying distribution of our data, we use approximations.
In particular, we examine \emph{Kernel Density Estimation}~(KDE), also called \emph{Parzen windowing}, similar to~\cite{bauckhage2020}
\begin{equation*}
	p_{\mathcal D}(\vec x)=\frac1n\sum_{\vec y\in\mathcal D}K(\vec x,\vec y)\;,
\end{equation*}
where $K:\Xdata\times\Xdata\to \mathbb R$ is a kernel function.
In~\cite{lehn2005vector,xu2008reproducing}, an RBF kernel is used to approximate $p_{\mathcal D}(\cdot)$ and $p_{\mathcal W}(\cdot)$ and the Cauchy-Schwartz divergence is minimized.
However, we are not restricted in the choice of the used kernel.
Due to Mercer's theorem~\cite{mercer1909xvi}, there exists a feature map $\phi:\Xdata\to \mathbb R^D$, s.t.,
\begin{subequations}\label{eq:kde_means}
	\begin{align}
		p_{\mathcal D}(\vec x)&=\frac1n\sum_{\vec y\in\mathcal D}\phi(\vec x)^{\top}\phi(\vec y) 
		=\ \phi(\vec x)^{\top}\vec{\phi}_{\mathcal X}, \label{eq:kde_means:D}\\ 
		p_{\mathcal W}(\vec x)&=\frac1k\sum_{\vec y\in\mathcal W}\phi(\vec x)^{\top}\phi(\vec y)
		=\phi(\vec x)^{\top}\vec{\phi}_{\mathcal W}, \label{eq:kde_means:W}
	\end{align}
\end{subequations}
with feature mean vectors
\begin{equation*}
	\vec{\phi}_{\mathcal D}\defeq\frac1n\sum_{\vec x\in\mathcal D}\phi(\vec x),\quad \vec{\phi}_{\mathcal W}\defeq\frac1k\sum_{\vec x\in\mathcal W}\phi(\vec x)\;.
\end{equation*}
The mean vectors $\vec{\phi}_{\mathcal X}$ and $\vec{\phi}_{\mathcal W}$ fully characterize the KDEs $p_{\mathcal D}(\cdot)$ and $p_{\mathcal W}(\cdot)$.
The difference between these vectors can be measured by the MMD~\cite{gretton2012kernel}, which leads to the following objective
\begin{subequations}\label{eq:md_vq}
	\begin{align}
		\min_{\mathcal W} &\ \left\|\vec{\phi}_{\mathcal D}-\vec{\phi}_{\mathcal W}\right\|^2 \\
		\text{s.t.}&\ |\mathcal W|=k\;.
	\end{align}
\end{subequations}
It is worth noting that prototypes coinciding with actual data points are often easier to interpret.
This makes KDE-VQ very similar to $k$-medoids-VQ, and we will investigate this connection later on.

We rewrite the MMD in terms of inner products
\begin{align*}
	&\min_{\mathcal W} \ \left\|\vec{\phi}_{\mathcal D}-\vec{\phi}_{\mathcal W}\right\|^2 \\
	\Leftrightarrow&\min_{\mathcal W}\ \vec{\phi}_{\mathcal W}^{\top}\vec{\phi}_{\mathcal W}-2\vec{\phi}_{\mathcal D}^{\top}\vec{\phi}_{\mathcal W}+\vec{\phi}_{\mathcal D}^{\top}\vec{\phi}_{\mathcal D}\\
	\Leftrightarrow&\min_{\mathcal W}\ \frac{1}{k^2}\sum_{\vec x\in\mathcal W}\sum_{\vec y\in\mathcal W}K(\vec x,\vec y)- \frac{2}{kn}\sum_{\vec x\in\mathcal W}\sum_{\vec y\in\mathcal D}K(\vec x,\vec y)
	\;,
\end{align*}
with $K(\vec x,\vec y)=\phi(\vec x)^{\top}\phi(\vec y)$ since $K$ is a Mercer kernel.

Again, letting the binary variable $z_i$ indicate whether $\vec x_i$ is an element of $\mathcal W\subset \mathcal D$, we obtain a QBP formulation equivalent to~\cref{eq:md_vq} by multiplying by $k^2$
\begin{subequations}
	\label{eq:qbp_kde}
\begin{align}
	\min_{\vec{z}\in\mathbb B^n}&\ \vec z^{\top}\mat K\vec z-\frac{2k}{n}(\mat K \vec 1)^{\top}\vec z \\
	\text{s.t.}&\ \vec 1^{\top}\vec z=k\;.
\end{align}
\end{subequations}

An equivalent QUBO formulation is then given by incorporating a weighted penalty term with large enough $\lambda_{\text{KDE}}>0$
\begin{equation}
	\min_{\vec{z}\in\mathbb B^n}\ \vec z^{\top}\mat K\vec z-\frac{2k}{n}(\mat K \vec 1)^{\top}\vec z+\lambda_{\text{KDE}}\left( \vec 1^{\top}\vec z-k\right)^2\;.
	\label{eq:qubo_kde}
\end{equation}
Similar to \cref{eq:kmedoids_qubo_matrix}, the QUBO matrix of \cref{eq:qubo_kde} is given by
\begin{equation}
	\mat{Q}_{\text{KDE}}=\mat K+\lambda_{\text{KDE}}\vec 1\vec 1^{\top}-2\diag\left[\frac{k}{n} \mat K\vec 1 +\lambda_{\text{MED}}k\vec 1\right]\;.
	\label{eq:kde_qubo_matrix}
\end{equation}

\section{Relationship between MED and KDE}

Interestingly, there is a connection between $k$-medoids-VQ and KDE-VQ.
It holds under the mild assumption of using normalized kernel functions.
\begin{proposition}\label{prop:kernel_distance}
	Let $K:\Xdata\times \Xdata\to \mathbb R$ be a Mercer kernel and $\phi:\mathbb X\to\mathbb R^d$ the underlying feature map, that is $\phi(\vec x)^{\top}\phi(\vec y)=K(\vec x, \vec y)$. 
	If all feature vectors lie on the unit sphere, i.e., $\left\|\phi(\vec x)\right\|=1$, $\forall \vec x\in\mathbb X$---also called normalized kernel $K(\vec x,\vec x)=1$---we can obtain a distance measure $D:\Xdata\times \Xdata\to \mathbb R$ by
	\begin{equation}
		D(\vec x, \vec y)=1-K(\vec x, \vec y)\;.
	\end{equation}
\end{proposition}
\begin{proof}
	Consider the squared euclidean distance between feature vectors $\phi(\vec x)$ and $\phi(\vec y)$
	\begin{align*}
		&\left\|\phi(\vec x)-\phi(\vec y)\right\|^2 \\
		=&\phi(\vec x)^{\top}\phi(\vec x)+\phi(\vec y)^{\top}\phi(\vec y)-2\phi(\vec x)^{\top}\phi(\vec y) \\
		=&2-2K(\vec x, \vec y)\;.
	\end{align*}
	Dividing the RHS by $2$, we obtain the claim.
\end{proof}
Examples for normalized kernels are RBF kernels, Laplacian kernels or quantum kernels.
Using \cref{prop:kernel_distance}, we first investigate the QUBO formulations.
\begin{theorem}\label{theo:qubo_connection}
Let $K:\Xdata\times \Xdata\to \mathbb R$ be a normalized Mercer kernel, $\mathcal D=\{\vec{x}^1,\dots,\vec{x}^n\}\subset \mathbb \Xdata$, $k<n$ and $\mat K$ be the corresponding kernel matrix. .
Defining the matrix $\mat D\defeq\vec 1\vec 1^{\top}-\mat K$, the QUBO matrices in~\cref{eq:kmedoids_qubo_matrix} and~\cref{eq:kde_qubo_matrix} are equal if we set $\gamma=\frac{2k}{n}$ and $\lambda_{\text{KDE}}=\lambda_{\text{MED}}+1$.
\end{theorem}
\begin{proof}
	Inserting $\gamma$ and $\lambda_{\text{KDE}}$ into the QUBO matrix definition in~\cref{eq:kmedoids_qubo} leads to 
	\begin{align*}
		&-\mat D+\lambda_{\text{KDE}} \vec 1\vec 1^{\top}+\diag\left[\gamma \mat D\vec 1-2\lambda_{\text{KDE}} k\vec 1\right] \\
		\Leftrightarrow\ & \mat K +(\lambda_{\text{KDE}}-1)\vec 1\vec 1^{\top}+\diag\left[-\gamma \mat K\vec 1 -(2\lambda_{\text{KDE}} k-n\gamma)\vec 1\right]\\
		\Leftrightarrow\ & \mat K +\lambda_{\text{MED}}\vec 1\vec 1^{\top} \\
		&\quad+\diag\left[-\frac{2k}{n} \mat K\vec 1 -\left(2\left(\lambda_{\text{MED}}+1\right)k-2\frac{kn}{n}\right)\vec 1\right] \\
		\Leftrightarrow\ & \mat K +\lambda_{\text{MED}}\vec 1\vec 1^{\top}-2\diag\left[\frac{k}{n} \mat K\vec 1 +\lambda_{\text{MED}}k\vec 1\right]
		\;,
	\end{align*}
	which corresponds to the QUBO matrix in~\cref{eq:qubo_kde}.
\end{proof}
The proof of~\cref{theo:qubo_connection} can be used to show equivalence between the QBP formulations.
\begin{corollary}\label{cor:qbp_connection}
	Let $\mat K$ and $\mat D$ be defined as in \cref{theo:qubo_connection}. Setting $\gamma=\frac{2k}{n}$, we obtain equivalence between the QBPs given in \cref{eq:kmedoids_qbp} and \cref{eq:qbp_kde}.
\end{corollary}

Using~\cref{theo:qubo_connection} and~\cref{cor:qbp_connection}, we can see that KDE-VQ is a special case of MED-VQ by choosing parameters in a specific way and using distance measures of the form $D(\vec x,\vec y)=1-K(\vec x,\vec y)$ for some Mercer kernel.
With the commonly used RBF kernel with a bandwidth of $2$, one exactly recovers Welsch's $M$-estimator, which was used as a distance function in \cite{bauckhage2019}.

\section{Discussion}
\label{sec:discussion}

Our findings highlight a surprising equivalence between the QUBO formulations of MED-VQ and KDE-VQ, revealing a deeper structural relationship that had not been explicitly articulated in prior work.
In particular, we showed that under a mild assumption on the kernel's feature map---specifically, that data points are mapped onto the unit sphere---the KDE-VQ QUBO becomes a special case of the MED-VQ QUBO with appropriately chosen parameters.

A key insight arising from this equivalence is the interpretation of the weighting parameter $\gamma$ in the k-medoids QUBO formulation. Traditionally, $\gamma$ has been viewed as a heuristic balancing factor, mediating between selecting central and mutually distant data points. 
Our work refines this perspective by showing that $\gamma$ can be interpreted as a geometric scaling factor within the kernel-induced feature space. 
This means that tuning $\gamma$ effectively scales the geometry of the prototype selection problem: it stretches or compresses the influence of centrality versus diversity in the feature space. 
In practical terms, small values of $\gamma$ prioritize selecting prototypes that are maximally distant from each other, while large  $\gamma$ values emphasize proximity to the dataset's density center.

This interpretation not only offers a conceptual understanding of $\gamma$ but also has implications for the implementation of these QUBO formulations on quantum-inspired or quantum-native hardware. 
Since QA and Ising machines rely on energy landscapes defined by QUBO matrices, the role of $\gamma$ translates directly into how the energy landscape is shaped in quantum optimization.
That is, it affects the spectral gap size of the problem Hamiltonian, which directly relates to the solvability in terms of needed time in AQC.
Thus, carefully choosing $\gamma$ can be seen as calibrating the kernel-induced geometry to align with the hardware’s optimization capabilities, which could be particularly relevant NISQ devices.

In summary, this geometric interpretation of the weighting parameter deepens our understanding of the prototype selection problem in VQ.
It also opens avenues for principled parameter tuning and hardware-specific adaptations in future work, bridging the gap between classical, kernel-based clustering and quantum optimization frameworks.

\section{Conclusion}
\label{sec:conclusion}

In this work, we demonstrated that the Quadratic Unconstrained Binary Optimization~(QUBO) formulation for Kernel Density Estimation-based Vector Quantization~(VQ) is a special case of the more general $k$-medoids-VQ QUBO formulation. This surprising equivalence provides a new lens through which to view prototype selection: as a unified framework that blends distance-based and density-based clustering approaches. 
By interpreting the weighting parameter as a geometric scaling factor within the kernel-induced feature space, we offer new insights into tuning and interpreting these models---especially in the context of quantum-inspired and quantum-native optimization. 
These findings pave the way for further exploration of how quantum algorithms and hardware can leverage these insights for efficient and interpretable VQ.

\section*{Acknowledgments}

This research has been funded by the Federal Ministry of Education
and Research of Germany and the state of North-Rhine Westphalia
as part of the Lamarr Institute for Machine Learning and Artificial
Intelligence.


\printbibliography

@book{gersho2012vector,
	title={Vector quantization and signal compression},
	author={Gersho, Allen and Gray, Robert M},
	year={2012},
	publisher={Springer Science \& Business Media}
}

@article{kohonen2001learning,
	title={Learning vector quantization},
	author={Kohonen, Teuvo and Kohonen, Teuvo},
	journal={Self-Organizing Maps},
	pages={245--261},
	year={2001},
	publisher={Springer}
}

@inproceedings{macqueen1967some,
	title={Some methods for classification and analysis of multivariate observations},
	author={MacQueen, James},
	booktitle={Proceedings of the Fifth Berkeley Symposium on Mathematical Statistics and Probability},
	volume={5},
	pages={281--298},
	year={1967},
	organization={University of California Press}
}

@article{cheng1995mean,
	title={Mean shift, mode seeking, and clustering},
	author={Cheng, Yizong},
	journal={IEEE Transactions on Pattern Analysis and Machine Intelligence},
	volume={17},
	number={8},
	pages={790--799},
	year={1995},
	publisher={IEEE}
}

@article{lisboa2023coming,
	title={The coming of age of interpretable and explainable machine learning models},
	author={Lisboa, Paulo JG and Saralajew, Sascha and Vellido, Alfredo and Fern{\'a}ndez-Domenech, Ricardo and Villmann, Thomas},
	journal={Neurocomputing},
	volume={535},
	pages={25--39},
	year={2023},
	publisher={Elsevier}
}

@article{vellido2020importance,
	title={The importance of interpretability and visualization in machine learning for applications in medicine and health care},
	author={Vellido, Alfredo},
	journal={Neural Computing and Applications},
	volume={32},
	number={24},
	pages={18069--18083},
	year={2020},
	publisher={Springer}
}

@inproceedings{kim2017interpretable,
	title={Interpretable learning for self-driving cars by visualizing causal attention},
	author={Kim, Jinkyu and Canny, John},
	booktitle={Proceedings of the 2017 International Conference on Computer Vision (ICCV)},
	pages={2942--2950},
	year={2017},
	publisher={IEEE}
}

@article{tshitoyan2019unsupervised,
	title={Unsupervised word embeddings capture latent knowledge from materials science literature},
	author={Tshitoyan, Vahe and Dagdelen, John and Weston, Leigh and Dunn, Alexander and Rong, Ziqin and Kononova, Olga and Persson, Kristin A and Ceder, Gerbrand and Jain, Anubhav},
	journal={Nature},
	volume={571},
	number={7763},
	pages={95--98},
	year={2019},
	publisher={Nature Publishing Group}
}

@article{lehn2005vector,
	title={Vector quantization using information theoretic concepts},
	author={Lehn-Schi{\o}ler, Tue and Hegde, Anant and Erdogmus, Deniz and Principe, Jose C},
	journal={Natural Computing},
	volume={4},
	number={1},
	%	pages={39--51},
	pages={39},
	year={2005},
	publisher={Springer}
}

@inproceedings{rao2007information,
	title={Information theoretic vector quantization with fixed point updates},
	author={Rao, Sudhir and Han, Seungju and Principe, Jose},
	booktitle={2007 International Joint Conference on Neural Networks (IJCNN)},
	pages={1020--1024},
	year={2007},
	organization={IEEE}
}

@article{xu2008reproducing,
	title={A reproducing kernel Hilbert space framework for information-theoretic learning},
	author={Xu, Jian-Wu and Paiva, Ant{\'o}nio RC and Park, Il and Principe, Jose C},
	journal={IEEE Transactions on Signal Processing},
	volume={56},
	number={12},
	pages={5891--5902},
	year={2008},
	publisher={IEEE}
}

@article{lloyd1982least,
	title={Least squares quantization in PCM},
	author={Lloyd, Stuart P},
	journal={IEEE Transactions on Information Theory},
	volume={28},
	number={2},
	pages={129--137},
	year={1982},
	publisher={IEEE}
}

@article{kaufman1987clustering,
	title={Clustering by means of medoids},
	author={Kaufman, Leonard and Rousseeuw, Peter J},
	journal={Proceedings of the First International Conference on Statistical Data Analysis Based on the L1-Norm},
	pages={405--416},
	year={1987},
	publisher={Elsevier Science Ltd}
}

@article{aloise2009np,
	title={NP-hardness of Euclidean sum-of-squares clustering},
	author={Aloise, Daniel and Deshpande, Amit and Hansen, Pierre and Popat, Preyas},
	journal={Machine learning},
	volume={75},
	pages={245--248},
	year={2009},
	publisher={Springer}
}

@book{punnen2022quadratic,
	title={The quadratic unconstrained binary optimization problem},
	author={Punnen, Abraham P},
	publisher={Springer},
	year={2022}
}

@article{mohseni2022ising,
	title={Ising machines as hardware solvers of combinatorial optimization problems},
	author={Mohseni, Naeimeh and McMahon, Peter L and Byrnes, Tim},
	journal={Nature Reviews Physics},
	volume={4},
	number={6},
	%	pages={363--379},
	pages={363},
	year={2022},
	publisher={Nature Publishing Group UK London}
}

@book{nielsen2010quantum,
	title={Quantum computation and quantum information},
	author={Nielsen, Michael A and Chuang, Isaac L},
	year={2010},
	publisher={Cambridge University Press}
}

@article{preskill2018quantum,
	title={Quantum computing in the NISQ era and beyond},
	author={Preskill, John},
	journal={Quantum},
	volume={2},
	pages={79},
	year={2018},
	publisher={Verein zur F{\"o}rderung des Open Access Publizierens in den Quantenwissenschaften}
}

@inproceedings{bauckhage2019,
	title={A QUBO formulation of the $k$-medoids problem},
	author={Bauckhage, Christian and Piatkowski, Nico and Sifa, Rafet and Hecker, Dirk and Wrobel, Stefan},
	booktitle={Proceedings of the Conference on “Lernen, Wissen, Daten, Analysen” (LWDA)},
	%	pages={54--63},
	pages={54},
	year={2019},
	publisher={CEUR-WS.org}
}

@inproceedings{bauckhage2020,
	title = {Hopfield {{networks}} for {{vector quantization}}},
	booktitle = {Proceedings of the 29th International Conference on Artificial Neural Networks (ICANN)},
	author = {Bauckhage, C. and Ramamurthy, R. and Sifa, R.},
	year = {2020},
	%	pages = {192--203},
	pages = {192},
	publisher = {{Springer}},
	%	doi = {10.1007/978-3-030-61616-8_16}
}

@article{havlivcek2019supervised,
	title={Supervised learning with quantum-enhanced feature spaces},
	author={Havlíček, Vojtěch and Córcoles, Antonio D and Temme, Kristan and Harrow, Aram W and Kandala, Abhinav and Chow, Jerry M and Gambetta, Jay M},
	journal={Nature},
	volume={567},
	number={7747},
	pages={209--212},
	year={2019},
	publisher={Nature Publishing Group}
}

@article{mercer1909xvi,
	title={Functions of positive and negative type, and their connection the theory of integral equations},
	author={Mercer, James},
	journal={Philosophical Transactions of the Royal Society of London},
	volume={209},
	%	number={441-458},
	%	pages={415--446},
	pages={415},
	year={1909},
	publisher={The Royal Society London}
}

@article{gretton2012kernel,
	title={A kernel two-sample test},
	author={Gretton, Arthur and Borgwardt, Karsten M and Rasch, Malte J and Sch{\"o}lkopf, Bernhard and Smola, Alexander},
	journal={The Journal of Machine Learning Research},
	volume={13},
	number={1},
	%	pages={723--773},
	pages={723},
	year={2012},
	publisher={JMLR. org}
}

@article{kullback1951information,
	title={On information and sufficiency},
	author={Kullback, Solomon and Leibler, Richard A},
	journal={The Annals of Mathematical Statistics},
	volume={22},
	number={1},
	%	pages={79--86},
	pages={79},
	year={1951},
	publisher={JSTOR}
}

@article{jenssen2006cauchy,
	title={The Cauchy--Schwarz divergence and Parzen windowing: Connections to graph theory and Mercer kernels},
	author={Jenssen, Robert and Principe, Jose C and Erdogmus, Deniz and Eltoft, Torbj{\o}rn},
	journal={Journal of the Franklin Institute},
	volume={343},
	number={6},
	%	pages={614--629},
	pages={614},
	year={2006},
	publisher={Elsevier}
}

@article{alessandroni2023alleviating,
	title={Alleviating the quantum Big-$ M $ problem},
	author={Alessandroni, Edoardo and Ramos-Calderer, Sergi and Roth, Ingo and Traversi, Emiliano and Aolita, Leandro},
	journal={arXiv preprint arXiv:2307.10379},
	year={2023}
}

@article{fisher1921probable,
	title={On the" probable error" of a coefficient of correlation deduced from a small sample},
	author={Fisher, Ronald A},
	journal={Metron},
	volume={1},
	pages={3--32},
	year={1921}
}

@article{kadowaki1998quantum,
	title={Quantum annealing in the transverse Ising model},
	author={Kadowaki, Tadashi and Nishimori, Hidetoshi},
	journal={Physical Review E},
	volume={58},
	number={5},
	pages={5355},
	year={1998},
	publisher={APS}
}

@article{albash2018adiabatic,
	title={Adiabatic quantum computation},
	author={Albash, Tameem and Lidar, Daniel A},
	journal={Reviews of Modern Physics},
	volume={90},
	number={1},
	pages={015002},
	year={2018},
	publisher={APS}
}

@inproceedings{villmann2021quantum,
	title={Quantum-hybrid neural vector quantization--a mathematical approach},
	author={Villmann, Thomas and Engelsberger, Alexander},
	booktitle={Proceedings of the 21st International Conference on Artificial Intelligence and Soft Computing (ICAISC)},
	pages={246--257},
	year={2021},
	organization={Springer}
}

@inproceedings{engelsberger2022steps,
	title={Steps forward to quantum learning vector quantization for classification learning on a theoretical quantum computer},
	author={Engelsberger, Alexander and Schubert, Ronny and Villmann, Thomas},
	booktitle={Proceedings of the 14th International Workshop on Self-Organizing Maps (WSOM)},
	pages={63--73},
	year={2022},
	organization={Springer}
}

@inproceedings{villmann2020quantum,
	title={Quantum-Inspired Learning Vector Quantization for Classification Learning.},
	author={Villmann, Thomas and Ravichandran, Jensun and Engelsberger, Alexander and Villmann, Andrea and Kaden, Marika},
	booktitle={Proceedings of the 28th European Symposium on Artificial Neural Networks (ESANN)},
	pages={279--284},
	year={2020}
}

@inproceedings{engelsberger2023quantum,
	title={Quantum-ready vector quantization: Prototype learning as a binary optimization problem},
	author={Engelsberger, Alexander and Villmann, Thomas},
	booktitle={Proceedings of the 31st European Symposium on Artificial Neural Networks (ESANN)},
%	pages={i6doc},
	year={2023}
}

@inproceedings{sijpesteijn2023quantum,
	title={Quantum approaches for medoid clustering},
	author={Sijpesteijn, Thom and Phillipson, Frank},
	booktitle={Proceedings of the 23rd International Conference on Innovations for Community Services (I4CS)},
	pages={222--235},
	year={2023},
	organization={Springer}
}

@article{laughhunn1970,
	title = {Quadratic binary programming with application to capital-budgeting problems},
	author = {Laughhunn, DJ},
	year = {1970},
	journal = {Operations research},
	volume = {18},
	number = {3},
	%	pages = {454--461},
	pages = {454},
	publisher = {{INFORMS}}
}

@inproceedings{gerlach2024fpga2,
	author={Gerlach, Thore and Knipp, Stefan and Biesner, David and Emmanouilidis, Stelios and Hauber, Klaus and Piatkowski, Nico},
	booktitle={Proceedings of the 2024 IEEE International Conference on Quantum Computing and Engineering (QCE)}, 
	title={Quantum optimization for FPGA-placement}, 
	year={2024},
	pages={637-647},
	publisher={IEEE}
}

@inproceedings{gerlach2025hybrid,
	title={Hybrid quantum-classical multi-agent pathfinding},
	author={Gerlach, Thore and Lee, Loong Kuan and Barbaresco, Fr{\'e}d{\'e}ric and Piatkowski, Nico},
	booktitle={Proceedings of the 42nd International Conference on Machine Learning (ICML)},
	year={2025},
	publisher={PMLR}
}

@inproceedings{gerlach2025dynamic,
	title={Dynamic range reduction via branch-and-bound},
	author={Gerlach, Thore and Piatkowski, Nico},
	booktitle={Proceedings of the 2025 IEEE International Conference on Quantum Computing and Engineering (QCE)},
	year={2025},
	publisher={IEEE}
}

@inproceedings{piatkowski2022towards,
	title={Towards bundle adjustment for satellite imaging via quantum machine learning},
	author={Piatkowski, Nico and Gerlach, Thore and Hugues, Romain and Sifa, Rafet and Bauckhage, Christian and Barbaresco, Frederic},
	booktitle={Proceedings of the 25th International Conference on Information Fusion (FUSION)},
	%	pages={1--8},
	pages={1},
	year={2022},
	%	volume={25},
	publisher={IEEE}
}

@article{mucke2025optimum,
	title={Optimum-preserving QUBO parameter compression},
	author={M{\"u}cke, Sascha and Gerlach, Thore and Piatkowski, Nico},
	journal={Quantum Machine Intelligence},
	volume={7},
	number={1},
	%	pages={1--18},
	pages={1},
	year={2025},
	publisher={Springer}
}

\end{document}